\documentclass[twoside,11pt]{article}

\usepackage{jmlr2e}
\usepackage{amsmath, amssymb, multirow, paralist}

\usepackage{algorithm,algorithmic}
\usepackage{color}
\usepackage{booktabs}
\usepackage{graphicx} % more modern
\usepackage[breaklinks]{hyperref}
\hypersetup{
colorlinks=true,
citecolor=blue,
linkcolor = {red},
linkbordercolor={1 1 1}, % set to white
citebordercolor={1 1 1} % set to white
}

\firstpageno{1}

\def \R {\mathbb{R}}
\def \D {\mathcal{D}}
\def \y {\mathbf{y}}

\def \E {\mathrm{E}}
\def \x {\mathbf{x}}
\def \a {\mathbf{a}}
\def \L {\mathcal{L}}
\def \H {\mathcal{H}}
\def \sp {\mbox{span}}

\def \Hk {\H_{\kappa}}

\def \v {\mathbf{v}}

\def \X {\mathcal{X}}

\def \z {\mathbf{z}}

\def \gh {\widehat{g}}

\def \et {\widetilde{\varepsilon}}

\def \xh {\widehat{\x}}
\def \Kh {\widehat{K}}

\def \sp {\mbox{span}}

\def \Er {\mathcal{E}}
\def \u {\mathbf{u}}

\def \v {\mathbf{v}}
\def \w {\mathbf{w}}
\def \Hb {\overline{\H}}
\def \Dh {\widehat{\D}}

\def \R {\mathbb{R}}
\def \Kt {\widetilde{K}}
\def \G {\mathcal{G}}
\def \I {{I}}

\begin{document}

\title{Improved Bounds for the Nystr\"{o}m Method with\\ Application to Kernel Classification}

 % Authors with different addresses:
\author{\name Rong Jin$^1$ \email rongjin@cse.msu.edu \\
\name   Tianbao Yang$^1$ \email yangtia1@msu.edu \\
       \name Mehrdad Mahdavi$^1$ \email mahdavim@msu.edu\\
        \name Yu-Feng Li$^2$ \email liyf@lamda.nju.edu.cn\\
              \name Zhi-Hua Zhou$^2$ \email zhouzh@nju.edu.cn\\
        \addr $^1$Department of Computer Science and Engineering\\
       Michigan State University, East Lansing, MI  48824, USA\\
               \addr $^2$National Key Laboratory for Novel Software Technology\\
               Nanjing University, Nanjing 210046, China\\
}

\editor{}

\maketitle

\begin{abstract}

We develop two approaches for analyzing the approximation error bound for the Nystr\"{o}m method, one based on the concentration inequality of integral operator, and one based on the compressive sensing theory. We show that the approximation error, measured in the spectral norm, can be improved from $O(N/\sqrt{m})$ to $O(N/m^{1 - \rho})$ in the case of large eigengap, where $N$ is the total number of data points, $m$ is the number of sampled data points, and $\rho \in (0, 1/2)$ is a positive constant that characterizes the eigengap. When the eigenvalues of the kernel matrix follow a $p$-power law, our analysis based on compressive sensing theory further improves the bound to $O(N/m^{p - 1})$ under an incoherence assumption, which explains why the Nystr\"{o}m method works well for kernel matrix with skewed eigenvalues. We present a kernel classification approach based on the Nystr\"{o}m method and derive its generalization performance using the improved bound. We show that when the eigenvalues of kernel matrix follow a $p$-power law, we can reduce the number of support vectors to $N^{2p/(p^2 - 1)}$, a number less than $N$ when $p > 1+\sqrt{2}$, without seriously sacrificing its generalization performance.

%\boldmath
\end{abstract}

%\begin{keywords}
%Nystr\"{o}m method, approximation error, concentration inequality, compressive sensing
%\end{keywords}

\section{Introduction}

The Nystr\"{o}m method has been widely applied in machine  learning to approximate large kernel matrices to speed up kernel algorithms~\citep{Williams01usingthe,Drineas05onthe,Fowlkes04spectralgrouping,kuma-2009-sampling,silva-2003-gloal,Platt04fastembedding,talwalkar-2008-large,kai-2008-improved,belabbas-2009-spectral,talwalkar-2010-matrix,cortes-2010-nystrom}. In order to evaluate the quality of the Nystr\"{o}m method, we typically bound the norm of the difference between the original kernel matrix and the low rank approximation created by the Nystr\"{o}m method. Several analysis were developed to bound the approximation error of the Nystr\"{o}m method~\citep{Drineas05onthe,kuma-2009-sampling,belabbas-2009-spectral,li2010icml,talwalkar-2010-matrix, NIPS2011_0669, journals/corr/abs-1110-5305}. Most of them focus on additive error bound, and base their analysis on the theoretical results from~\citep{Drineas05onthe}. When the target matrix is of low rank, significantly better bounds for the approximation error of the Nystr\"{o}m method were given in~\citep{talwalkar-2010-matrix} and~\citep{NIPS2011_0669}. They are further generalized to kernel matrix of an arbitrary rank by a relative error bound in~\citep{journals/corr/abs-1110-5305}. Although a relative error bound is usually tighter than an additive bound~\citep{DBLP:journals/ftml/Mahoney11}, the relative error bound in~\citep{journals/corr/abs-1110-5305} is proportional to $N$, where $N$ is the total number of data points, making it unattractive for kernel matrix of very large size. In this study, we focus on the additive error bound of the Nystr\"{o}m method for general matrices, and will compare our results mainly to the ones stated in~\citep{Drineas05onthe}~\footnote{For completeness, we did include the comparison to the relative error bound in~\citep{journals/corr/abs-1110-5305} in the later remarks. }. Below, we review the main results in~\citep{Drineas05onthe} and their limitations.

Let $K \in \R^{N\times N}$ be the kernel matrix to be approximated, and $\lambda_i, i=1, \ldots, N$ be the eigenvalues of $K$ ranked in the descending order. Let $\widetilde K(r)$ be an approximate kernel matrix of rank $r$ generated by the Nystr\"{o}m method. Let $m$ be the number of columns sampled from $K$ used to construct $\widetilde K(r)$. Then, under the assumption $K_{i,i} = O(1)$, %$K_{i,i} = \mbox{constant}$, 
\citet{Drineas05onthe} showed that for any {$m$ uniformly sampled columns}~\footnote{Although the main results in~\citep{Drineas05onthe} use a data dependent sampling scheme, it was stated in the original paper that the results also hold for uniform sampling.},  with a high probability,
\begin{equation*}
    \|K - \widetilde K(r)\|_2 \leq \lambda_{r+1} + O\left(\frac{N}{\sqrt{m}}\right),
\end{equation*}
where $\|\cdot\|_2$ stands for the spectral norm of a matrix. %When $K$ is non-singular,
By setting $r = m$, the bound in (\ref{eqn:bound-1}) becomes
\begin{eqnarray}
    \|K - \widetilde K(m)\|_2 \leq \lambda_{m+1} + O\left(\frac{N}{\sqrt{m}}\right). \label{eqn:bound-1}
\end{eqnarray}

The main problem with the bound in (\ref{eqn:bound-1}) is its slow reduction rate in the number of sampled columns (i.e., $O(m^{-1/2})$), implying that a large number of samples is needed in order to achieve a small approximation error.
In this study, we aim to improve the approximation error bound in~(\ref{eqn:bound-1}) by considering two special cases of the kernel matrix $K$. In the first case, we assume there is a large eigengap in the spectrum of $K$. More specifically, we assume there exists a rank $r \in [N]$ such that $\lambda_r = \Omega(N/m^{\rho})$ and $\lambda_{r+1} = O(N/m^{1 - \rho})$, where $\rho < 1/2$. Here, parameter $\rho$ is introduce to characterize the eigengap $\lambda_{r} - \lambda_{r+1}$: the smaller the $\rho$, the larger the eigengap will be. We show that the approximation error bound is improved to $O(N/m^{1 - \rho})$ in the case of large eigengap. The second case assumes that the eigenvalues of $K$ follow a $p$-power law with $p > 1$. We show that the approximation error is improved to $O(N/m^{p - 1})$ provided that the eigenvector matrix satisfies an incoherence assumption~\footnote{A similar assumption was used in the previous analysis of the Nystr\"{o}m method~\citep{talwalkar-2010-matrix, NIPS2011_0669,journals/corr/abs-1110-5305}.}. This result explains why the Nystr\"{o}m method works well for kernel matrices with skewed eigenvalue distributions~\citep{talwalkar-2010-matrix}.

The second contribution of this study is a kernel classification algorithm that explicitly explores the improved bounds of the Nystr\"{o}m method developed here. We show that when the eigenvalues of the kernel matrix follow a $p$-power law with $p>1$, we can construct a kernel classifier that yields a similar generalization performance as the full version of kernel classifier but with no more than $N^{2p/(p^2 - 1)}$ support vectors, which is sublinear in $N$ when $p> (1+\sqrt{2})$. Although the generalization error bound of using the Nystr\"{o}m method for classification has been studied in~\citep{cortes-2010-nystrom}, to the best of knowledge, this is the first work that bounds the number of support vectors using the analysis of the Nystr\"{o}m method.

\section{Notations and Background}
Let $\D = \{\x_1, \ldots, \x_N\}$ be a collection of $N$ samples, where $\x_i\in\mathcal X$, and $K = [\kappa(\x_i, \x_j)]_{N\times N}$ be the kernel matrix for the samples in $\D$, where $\kappa(\cdot, \cdot)$ is a kernel function. For simplicity, we assume  $\kappa(\x, \x) \leq 1$ for any $\x \in \X$. % and (ii) $K$ is non-singular~\footnote{This is true for the RBF kernel function}.
We denote by $(\v_i, \lambda_i), i=1, \ldots, N$ the eigenvectors and eigenvalues of $K$ ranked in the descending order of eigenvalues, and by $V=(\v_1,\cdots, \v_N)$ the orthonormal eigenvector matrix. In order to build the low rank approximation of kernel matrix $K$, the Nystr\"{o}m method first samples $m < N$ examples randomly from $\D$, denoted by $\Dh = \left\{\xh_1, \ldots, \xh_m\right\}$. Let  $\Kh = [\kappa(\xh_i, \xh_j)]_{m\times m}$ measure the kernel similarity between any two samples in $\Dh$ and $K_b = [\kappa(\x_i, \xh_j)]_{N\times m}$ measure the similarity between the samples in $\D$ and $\Dh$. %Finally, we use $O(\cdot), \Omega(\cdot)$ to denote upper bound and lower bound, respectively.
Using the samples in $\Dh$, with rank $r$ set to $m$ (or the rank of $\Kh$ if it is less than $m$), the Nystr\"{o}m method approximates $K$ by $K_b\Kh^{\dagger}K_b^{\top}$, where $\Kh^\dagger$ denote the pseudo inverse of $\Kh$. Our goal is to provide a high probability bound for the approximation error $\left\|K - K_b\Kh^{\dagger}K_b^{\top}\right\|_2$. We choose $r = m$ (or the rank of $\Kh$) because according to~\citep{Drineas05onthe,kuma-2009-sampling}, it yields the best approximation error for a non-singular kernel matrix.

%which bound  the approximation error of rank-$r$ matrix $K_b\Kh_r^{-1}K_b^{\top}$, where $\Kh_r$ is the best rank-r approximation to $\Kh$, we directly bound the approximation error of $K_b\Kh^{-1}K_b^{\top}$.

In this study, we focus on the spectral norm for measuring the approximation error, which is particularly suitable for kernel classification~\citep{cortes-2010-nystrom}. We also restrict the analysis to the uniform sampling for the Nystr\"{o}m method. Although different sampling approaches have been suggested for the Nystr\"{o}m method~\citep{Drineas05onthe,kuma-2009-sampling,kai-2008-improved,belabbas-2009-spectral}, according to~\citep{kuma-2009-sampling}, for real-world datasets, uniform sampling is the most efficient and yields performance comparable to the other sampling approaches. We notice that in~\citep{belabbas-2009-spectral}, the authors show a significantly better approximation bound for the Nystr\"{o}m method when employing the determinantal process~\citep{hough-2006-determinantal} for column selection; however, it is important to point out that the determinantal process is usually computationally expensive as it requires computing the determinant of the submatrix for the selected columns/rows, making it unsuitable for the case when a large number of columns are needed to be sampled.

Our analysis for the Nystr\"{o}m method extensively exploits the properties of the integral operator. This is in contrast to most of the previous studies for the Nystr\"{o}m method that rely on matrix analysis. The main advantage of using the integral operator is its convenience in handling the unseen data points (i.e., test data), making it attractive for the analysis of generalization error bounds. In particular, we introduce a linear operator $L_N$ defined over the samples in $\D$. For any function $f(\cdot)$, operator $L_N$ is defined as
\[
    L_N[f](\cdot) = \frac{1}{N}\sum_{i=1}^N \kappa(\x_i, \cdot) f(\x_i).
\]
It can be shown that the eigenvalues of the  operator $L_N$ are $\lambda_i/N, i=1, \ldots, N$~\citep{smale-2009-geometry}. Let $\varphi_1(\cdot), \ldots, \varphi_N(\cdot)$ be the corresponding eigenfunctions of $L_N$ that are normalized by functional norm, i.e.,
$ \langle \varphi_i, \varphi_j \rangle_{\Hk} = \delta(i,j), 1 \leq i \leq j \leq N$, where $\langle \cdot , \cdot\rangle_{\Hk}$ denotes the inner product in $\Hk$.
According to~\citep{smale-2009-geometry}, the eigenfunctions satisfy
\begin{equation}
\sqrt{\lambda_j}\varphi_j(\cdot) =\sum_{i=1}^N V_{i,j} \kappa(\x_i, \cdot), j=1,\cdots, N,
\end{equation}
where $V_{i,j}$ is the $(i,j)$th element in $V$.
Similarly, we can write $\kappa(\x_j, \cdot)$ by its eigen-expansion as
\begin{equation}\label{eqn:keig}
    \kappa(\x_j, \cdot) = \sum_{i=1}^N  \sqrt{\lambda_i} V_{j, i}\varphi_i(\cdot), j=1, \ldots, N.
\end{equation}
Furthermore, let $L_m$ be an operator defined on the samples in $\Dh$, i.e.,
\[
    L_m[f](\cdot) = \frac{1}{m}\sum_{i=1}^m \kappa(\xh_i, \cdot) f(\xh_i).
\]
Finally we denote by $\langle f, g\rangle_{\Hk}$ and $\|f\|_{\Hk}$  the inner product and function norm in Hilbert space $\Hk$, respectively, and denote by $\|L\|_{HS}$ and $\|L\|_2$ the Hibert Schmid norm  and spectral  norm of a linear  operator $L$, respectively, i.e.
 \begin{align*}
\|L\|_{HS}&= \sqrt{\sum_{i, j} \langle \varphi_i, L\varphi_j\rangle_{\Hk}^2}\;\; \text{and}\;\;\|L\|_2 =  \max_{\|f\|_{\Hk}\leq 1}\|Lf\|_{\Hk},
\end{align*}
where $\{\varphi_i,i=1,\cdots,\}$ is a complete orthogonal basis of $\Hk$. The two norms  are the analogs of Frobenius and spectral norm in Euclidean space, respectively.  In the following analysis, omitted proofs are presented in the appendix.

 %%%
\section{Approximation Error Bound by the Nystr\"{o}m Method}
Our first step is to turn $\|K - K_b\Kh^{\dagger}K_b^{\top}\|_2$ into a functional approximation problem. To this end, we introduce two sets:
\begin{align*}
 \H_a &= \sp\left(\kappa(\xh_1, \cdot), \ldots, \kappa(\xh_m, \cdot) \right)\\
 \H_b &= \left\{f(\cdot) = \sum_{i=1}^N u_i \kappa(\x_i, \cdot): \sum_{i=1}^N u_i^2 \leq 1 \right\},
 \end{align*}
 where $\H_a$ is the subspace spanned by kernel functions defined on the samples in $\Dh$, and $\H_b$ is a subset of a functional space spanned by kernel functions defined on the samples in $\D$ with bounded coefficients.   Using the eigen-expansion of $\kappa(\x_j, \cdot)$ in~(\ref{eqn:keig}), it is straightforward to show that $\H_b$ can be rewritten in the basis of  the eigenfunctions $\{\varphi_i\}_{i=1}^N$
\[
    \H_b = \left\{f(\cdot) = \sum_{i=1}^N w_i\sqrt{\lambda_i} \varphi_i(\cdot): \sum_{i=1}^N w_i^2\leq 1 \right\}.
\]

 Define $\Er(g, \H_a)$ as the minimum error in approximating a function $g \in \H_b$ by functions in $\H_a$, i.e.,
\begin{align*}
\Er(g, \H_a) &= \min\limits_{f \in \H_a} \|f - g\|_{\Hk}^2\\
& = \|f\|_{\Hk}^2 + \|g\|_{\Hk}^2 - 2 \left\langle f, g\right\rangle_{\Hk}.
\end{align*}
Define $\Er(\H_a)$ as the worst error in approximating any function $g \in \H_b$ by functions in $\H_a$, i.e.,
\begin{eqnarray}
\Er(\H_a) = \max\limits_{g \in \H_b} \Er(g, \H_a).
\end{eqnarray}
The following proposition connects  $\left\|K - K_b\Kh^{\dagger}K_b^{\top}\right\|_2$ with $\Er(\H_a)$.
\begin{proposition}
For any random samples $\xh_1, \ldots, \xh_m$, we have
\[
     \left\|K - K_b\Kh^{\dagger}K_b^{\top}\right\|_2 = \Er(\H_a).
\]
\end{proposition}
\begin{proof}
Since $g \in \H_b$ and $f \in \H_a$, we can write $g$ and $f$ as
\[
    g = \sum_{i=1}^N u_i \kappa(\x_i, \cdot)\quad \text{and} \quad   f = \sum_{i=1}^m z_i \kappa(\xh_i, \cdot),
\]
where $\u = (u_1, \ldots, u_N)^{\top} \in \R^N$ satisfies $\|\u\|_2\leq 1$ and $\z = (z_1, \ldots, z_m)^{\top} \in \R^m$. We thus can rewrite $\Er(g, \H_a)$ as an optimization problem in terms of $\z$, i.e.,
\begin{align*}
\Er(g, \H_a) &= \min\limits_{\z \in \R^m} \z^{\top}\Kh\z - 2\u^{\top} K_b \z + \u^{\top}K\u\\
& = \u^{\top}\left(K - K_b\Kh^{\dagger} K_b^{\top}\right)\u,
\end{align*}
and therefore
\begin{align*}
\Er(\H_a) &= \max\limits_{g \in \H_b} \Er(g, \H_a) \\ &= \max\limits_{\|\u\|_2 \leq 1} \u^{\top}\left(K - K_b\Kh^{\dagger} K_b^{\top}\right)\u \\
&= \left\| K - K_b\Kh^{\dagger} K_b^{\top} \right\|_2.
\end{align*}
\end{proof}
\begin{remark}We can restrict the space $\H_a$ to its subspace $\displaystyle\H^r_a=\left\{\sum_{i=1}^mz_i\kappa(\xh_i,\cdot): \z\in \text{span}(\widehat\v_1,\ldots, \widehat\v_r)\right\}$, where $\widehat \v_i, i=1,\ldots, r$ are the first $r$ eigenvectors of $\Kh$, to conduct the analysis for the rank $r<m$  approximation of the Nystr\"{o}m method.
\end{remark}
%\paragraph{Remark} $\Er(\H_a)$ is closely related to Gel'fand n-width~\citep{pinkus-1985-nwidth}. Let $X \in \R^N$ be a bounded set. The Gel'fand n-width of $X$ with respect to $\ell_2$ norm is defined as
%\[
%    d^m(X, \ell_2) = \inf\limits_{V_m}\sup\limits_{\x \in X} \left\{\|\x\|_2: \x \in V_m^{\perp} \cap X \right\}
%\]
%where $V_m$ is any $m$-dimensional linear subspaces of $\R^N$, and $V_m^{\perp}$ denotes the orthocomplement of $V_m$ with respect to the standard Euclidean inner product. We have the following proposition that connects $\Er(\H_a)$ with $d^m(X)$.
%\begin{proposition}
%For any sample $\xh_1, \ldots, \xh_m$, we have
%\[
%    \Er(\H_a) \geq [d^m(\H_b, \|\cdot\|_{\Hk})]^2
%\]
%\end{proposition}
%\begin{proof}
%Since
%\[
%\Er(\H_a) = \max\limits_{g \in \H_b} \min\limits_{f \in \H_a} \|f - g\|_{\Hk}^2 \geq \sup\left\{\|g\|_{\Hk}^2: g \perp \H_a, g \in \H_b\right\}
%\]
%%we have
%%\[
%%    \Er(\H_a) \geq \min\limits_{\H_a} \sup\left\{\|g\|_{\Hk}^2: g \perp \H_a, g \in \H_b\right\}
%%\]
%Since the dimension of $\H_a$ is no larger than $m$, we have
%\[
%    \Er(\H_a) \geq \min\limits_{V_m} \sup\left\{\|g\|_{\Hk}^2: g \perp V_m, g \in \H_b \right\} = [d^m(\H_b, \|\cdot\|_{\Hk})]^2
%\]
%\end{proof}
To proceed our analysis, for any $r\in[N]$ we define
\begin{align*}
    &\H_r =\text{span}(\varphi_1(\cdot),\cdots, \varphi_{r}(\cdot)),\\
    &\Hb_r =\text{span}(\varphi_{r+1}(\cdot),\cdots, \varphi_{N}(\cdot)),\\
    &  \H_b^r = \left\{f(\cdot) = \sum_{i=1}^r w_i\sqrt{\lambda_i}\varphi_i(\cdot): \sum_{i=1}^r w_i^2\leq 1\right\},\\
     &\Hb_b^r = \left\{f(\cdot) = \sum_{i=1}^{N-r} w_i\sqrt{\lambda_{i+r}}\varphi_{i+r}(\cdot): \sum_{i=1}^{N-r}w_i^2 \leq 1\right\}.
    \end{align*}
   Define
$
    \Er(\H_a, r) = \max\limits_{g \in \H_b^r} \Er(g, \H_a)
$ as the worst error in approximating any function $g\in \H_b^r$ by functions in $\H_a$.
The proposition below bounds $\Er(\H_a)$ by $\Er(\H_a, r)$.
\begin{proposition} \label{prop:1} For any $r \in [N]$, we have
\[
    \Er(\H_a) \leq \max\left(\Er(\H_a, r), \lambda_{r+1}\right) \leq \Er(\H_a, r) + \lambda_{r+1}.
\]
\end{proposition}
\begin{proof}
We first note that for any $f\in \H_a$ can be written as $f=f_1+f_2$, where $f_1\in \H_a\cap \H_r$, and $f_2\in \H_a\cap \Hb_r$.
For any $g\in\H_b$, we can write $g=g_1+g_2$, where $g_1\in\sqrt{1-\delta}\H_b^r$, $g_2\in\sqrt{\delta}\Hb_b^r$, and $\delta \in [0, 1]$. Using these notations, we rewrite $\Er(\H_a)$ as
\begin{align*}
    &\Er(\H_a) \\
    & =   \max\limits_{\small \begin{array}{l}
    \delta\in[0,1]\\g_1 \in \sqrt{1 - \delta}\H_b^r \\
    g_2 \in \sqrt{\delta}\Hb_b^r
    \end{array}
    } \min\limits_{\small\begin{array}{l}
    f_1 \in \H_a \cap \H_r \\
    f_2 \in \H_a \cap \Hb_r
    \end{array}} \hspace*{-0.1in}\|f_1 - g_1\|^2 + \|f_2 - g_2\|_{\Hk}^2 \\
    & \leq  \max\limits_{\delta \in [0, 1]} (1 - \delta)\max\limits_{g \in \H_b^r} \min\limits_{f \in \H_a\cap \H_r} \|f - g\|_{\Hk}^2 + \delta \max\limits_{g \in \Hb_b^r} \|g\|_{\Hk}^2\\
        & =  \max\limits_{\delta \in [0, 1]} \left\{(1 - \delta)\max\limits_{g \in \H_b^r} \min\limits_{f \in \H_a} \|f - g\|_{\Hk}^2 + \delta \max\limits_{g \in \Hb_b^r} \|g\|_{\Hk}^2 \right\}\\
    & =  \max\limits_{\delta \in [0, 1]} (1 - \delta)\Er(\H_a, r) + \delta\lambda_{r+1} = \max\left( \Er(\H_a, r), \lambda_{r+1}\right),
\end{align*}
where the second equality follows that for any $g\in \H_b^r$, $\min\limits_{f\in\H_a}\|f-g\|_{\Hk}^2= \min\limits_{f\in \H_a\cap\H_r}\|f-g\|_{\Hk}^2$, and the last inequality follows the definition of $\Er(\H_a, r)$.
\end{proof}

As indicated by Proposition~\ref{prop:1}, in order to bound the approximation error $\Er(\H_a)$, we can  bound $\Er(\H_a, r)$, namely the approximation error for functions in the subspace spanned by the top eigenfunctions of $L_N$. In the next two subsections, we discuss two approaches for bounding $\Er(\H_a, r)$: the first approach relies on the concentration inequality of  integral operator~\citep{smale-2009-geometry}, and the second approach explores the compressive sensing theory~\citep{Candes:2007:Sparsity}. Before proceeding to upper bound $\Er(\H_a)$, we first provide a lower bound for $\Er(\H_a)$.

\begin{theorem}\label{thm:lb}
There exists a kernel matrix $K\in\mathbb R^{N\times N}$ with all its diagonal entries being $1$ such that for any sampling strategy that selects $m$ columns, the approximation error of the Nystr\"{o}m method is lower bounded by $\Omega(\frac{N}{m})$, i.e.,
\[
    \left\| K - K_b\Kh^{\dagger}K_b^{\top}\right\|_2 \geq \Omega\left(\frac{N}{m} \right),
\]
provided $N > 64[\ln 4]^2 m^2$.
\end{theorem}
\begin{remark}
Theorem~\ref{thm:lb}  shows that the lower bound for the approximation error of the Nystr\"{o}m method is $\Omega(N/m)$. The analysis developed in this work aims to bridge the gap between the known upper  bound (i.e., $O(N/\sqrt{m})$) and the obtained lower bound.
\end{remark}

\subsection{Bound for $\Er(\H_a, r)$  using Concentration Inequality of Integral Operator}
In this section, we bound $\Er(\H_a, r)$ using the concentration inequality of integral operator. We show that the approximation error of the Nystr\"{o}m method can be improved to $O(N/m^{1 - \rho})$ when there is a large eigengap in the spectrum of kernel matrix $K$, where $\rho < 1/2$ is introduced to characterize the eigengap. We first state the  concentration inequality of a general random variable. 

\begin{proposition}\label{prop:conc}(Proposition 1~\citep{smale-2009-geometry})
Let $\xi$ be a random variable on $(\X, P_{\X})$ with values in a Hilbert space $(\H,\|\cdot\|)$. Assume $\|\xi\| \leq M < \infty$ is almost sure. Then with a probability at least $1 - \delta$, we have
\[
    \left\|\frac{1}{m}\sum_{i=1}^m \xi(\x_i) - \E[\xi]\right\| \leq \frac{4M\ln(2/\delta)}{\sqrt{m}}.
\]
\end{proposition}

The approximation error of the Nystr\"{o}m method using the concentration inequality is given in the following theorem.
\begin{theorem}\label{thm:2} With a probability at least $1 - \delta$, for any $r \in [N]$, we have
\[
\left\| K - K_b\Kh^{\dagger} K_b^{\top}\right\|_2 \leq \frac{16[\ln(2/\delta)]^2 N^2}{m\lambda_r} + \lambda_{r+1}.
\]
\end{theorem}
We consider the scenario where there is very large eigengap in the spectrum of kernel matrix $K$. In particular, we assume that there exists a rank $r$ and $\rho \in (0, 1/2)$ such that $\lambda_r = \Omega(N/m^{\rho})$ and $\lambda_{r+1} = O(N/m^{1-\rho})$. Parameter $\rho$ is introduced to characterize the eigengap which is given by
\[
\lambda_r - \lambda_{r+1} = \Omega\left(\frac{N}{m^{\rho}} - \frac{N}{m^{1 - \rho}}\right) = \Omega\left(\frac{N}{m^{\rho}}\left[1 - \frac{1}{m^{1-2\rho }}\right]\right)
\]
Evidently, the smaller the $\rho$, the larger the eigengap. When $\rho = 1/2$, the eigengap is small. Under the large eigengap assumption, the bound in Theorem~\ref{thm:2} is simplified as
\begin{eqnarray}
    \left\| K - K_b\Kh^{\dagger} K_b^{\top}\right\|_2 \leq O\left(\frac{N}{m^{1 - \rho}}\right). \label{eqn:bb}
\end{eqnarray}
Compared to the bound in (\ref{eqn:bound-1}), the bound in (\ref{eqn:bb}) improves the approximation error from $O(N/\sqrt{m})$ to $O(N/m^{1 - \rho})$, when $\rho < 1/2$.

To prove Theorem~\ref{thm:2}, we define two sets of functions
\begin{align*}
\H^r_c &= \left\{h = \sum_{i=1}^r w_i \sqrt{\lambda_i}\varphi_i(\cdot): \frac{1}{N^2}\sum_{i=1}^r w_i^2\lambda_i^2 \leq 1 \right\},\\
\H_d^r &= \left\{f \in \Hk: \|f\|^2_{\Hk} \leq N^2/\lambda_r \right\}.
\end{align*}
where $r$ corresponds to the rank with a large eigengap. It is evident that $\H_c^r\subseteq \H_d^r$; and for any $g \in \H^r_b$, it can also be written as $g = L_N[h]$, where $h \in \H^r_c$.

Using $\H^r_c$ and $\H^r_d$, we have
\begin{align*}
\Er(\H_a, r) &= \max\limits_{g \in \H^r_b} \Er(g, \H_a) = \max\limits_{h \in \H^r_c} \min\limits_{f \in \H_a} \|L_N h - f\|_{\Hk}^2\\
& \leq \max\limits_{h \in \H^r_d} \min\limits_{f \in \H_a} \|L_N h - f\|_{\Hk}^2.
\end{align*}
By constructing $f$ as $L_m[h]$  we can bound  $\Er(\H_a, r)$  as
\begin{align}\label{eqn:diff}
\Er(\H_a, r)& \leq \max\limits_{h \in \H_d^r} \min\limits_{f \in \H_a} \|L_N(h) - f\|_{\Hk}^2\nonumber\\
& \leq \max\limits_{h \in H_d^r} \|(L_N - L_m) h\|_{\Hk}^2\nonumber\\
&\leq \|L_N-L_m\|^2_2\frac{N^2}{\lambda_r} \nonumber \\
& \leq \|L_N - L_m\|_{HS}^2\frac{N^2}{\lambda_r},
\end{align}
where the last step follows the fact $\|L_N - L_m\|_2 \leq \|L_N - L_m\|_{HS}$. The following corollary allows us to bound the difference between $L_N$ and $L_m$ and follows immediately from Proposition~\ref{prop:conc}.
\begin{corollary}\label{cor:1}
With a probability $1 - \delta$, we have
\[
\|L_N - L_m\|_{HS} \leq \frac{4\ln(2/\delta)}{\sqrt{m}}.
\]
\end{corollary}
Finally, Theorem~\ref{thm:2} follows directly the inequality in~(\ref{eqn:diff}) and the result in Corollary~\ref{cor:1}.

\subsection{Bound for $\Er(\H_a,r)$ using Compressive Sensing Theory}

In this subsection, we aim to develop a better error bound for the Nystr\"{o}m method for kernel matrices with eigenvalues that follow a power law distribution. Our analysis explicitly explores some of the key results in the theory of compressive sensing~\citep{Candes:2007:Sparsity,DBLP:journals/tit/Donoho06}.  To this end, we first  introduce the definition of the power law distribution of eigenvalues~\citep{Koltchinskii:2010:mkl,DBLP:conf/nips/KloftB11}.  The eigenvalues $\sigma_i, i=1,\ldots$ ranked in the non-increasing order follows a $p$-power law (distribution) if there exists constant $c > 0$ such that
\[
    \sigma_k \leq c k^{-p}.
\]

In the sequel, we assume the normalized eigenvalues  $\lambda_i/N, i=1,\ldots, N$ (i.e., the eigenvalues of the operator $L_N$), follow a $p$-power law distribution~\footnote{We assume a power law distribution for the normalize eigenvalues $\lambda_i/N$ because the eigenvalues $\lambda_i$ of $K$ scales in $N$.}. A well-known example of kernel with a power law eigenvalue distribution~\citep{Koltchinskii:2010:mkl} is the kernel function that generates Sobolev Spaces $W^{\alpha, 2}(\mathbb T^d)$ of smoothness $\alpha> d/2$, where $\mathbb T^d$ is $d$-dimensional torus. Its eigenvalues follow a $p$-power law with $p=2\alpha>d$. It is also observed that the eigenvalues of a Gaussian kernel by appropriately setting the width parameter  follow a power law distribution~\citep{Ming2012ICML}.

In order to exploit the compressive sensing theory~\citep{Candes:2007:Sparsity}, we introduce the definition of the coherence $\mu$ for the eigevenvector matrix $V = (\v_1, \ldots, \v_N)$ as
\[
\mu = \sqrt{N} \max\limits_{1 \leq i, j \leq N} |V_{i,j}|.
\]
Intuitively, the coherence measures the degree to which the eigenvectors in $V$ are correlated with the canonical bases.  According to the theory of compressive sensing,  highly coherent matrices are difficult (even impossible) to be recovered by matrix completion with random sampling.  As observed in previous studies~\citep{talwalkar-2010-matrix} and seen later in our analysis, the coherence of $V$ also plays an important role in measuring the approximation performance of the Nystr\"{o}m method using an uniform sampling.

The coherence measure was first introduced into the error analysis of the Nystr\"{o}m method by Talwalkar and Rostamizadeh~\citep{talwalkar-2010-matrix}. Their analysis shows that a low rank kernel matrix with incoherent eigvenvectors (i.e., with low coherence) can be accurately approximated by the Nystr\"{o}m method using an uniform sampling.  This result is generalized to noisy observation in~\citep{NIPS2011_0669} for low rank matrix. The main limitation of these results is that they only apply to low rank matrices. Recently, A. Gittens~\citep{journals/corr/abs-1110-5305} developed a relative error bound of the Nystr\"{o}m  method for kernel matrices with an arbitrary rank using a slightly different coherence measure.
Unlike the previous studies, we focus on the error bound of the Nystr\"{o}m method for kernel matrices with an arbitrary rank and a skewed eigenvalue distribution. The main result of our analysis is given in the following theorem.
\begin{theorem}\label{thm:comp}
Assume the eigenvalues $\lambda_i/N, i=1,\ldots, N$ follow a $p$-power law with $p > 1$. Given a sufficiently large number of samples, i.e.,
\begin{eqnarray*}
    m > \mu^2\max\left(16\left(\frac{\ln N}{\gamma}\right)^2, 2C_{ab}\ln(3N^3), 4C^2_{ab}\ln^2 (3N^3)\right)
\end{eqnarray*}
we have, with a probability $1 - 2N^{-3}$,
\[
    \left\|K - K_b\Kh^{\dagger}K_b^{\top} \right\|_2 \leq \widetilde{O}\left(\frac{N}{m^{p - 1}}\right),
\]
where $\widetilde{O}(\cdot)$ suppresses  the polynomial factor that depends on $\ln N$,  and $C_{ab}$ is a numerical constant as revealed in our later analysis.
\end{theorem}
\begin{remark} Compared to the approximation error in~(\ref{eqn:bound-1}), Theorem~\ref{thm:comp} improves the bound from $O(N/\sqrt{m})$ to $O(N/m^{p - 1})$ provided the eigenvalues of kernel matrix follow a power law. For the relative error bound given in~\citep{journals/corr/abs-1110-5305}, the approximation error is dominated by $O(N^2/m^{p+1})$ for eigenvalues following a $p$-power law. It is straightforward to see that the result in Theorem~\ref{thm:comp} is better than $O(N^2/m^{p+1})$ when $m \leq \sqrt{N}$, a favorable setting when $N$ is very large and $m$ is small. Finally it is worth noting that similar to~\citep{talwalkar-2010-matrix, NIPS2011_0669, journals/corr/abs-1110-5305}, the bound in Theorem~\ref{thm:comp} is meaningful only when the coherence $\mu$ of the eigenvector matrix is small (i.e., the eigenvector matrix satisfies the incoherence assumption).
\end{remark}
We emphasize that the result in Theorem~\ref{thm:comp} does not contradict the lower bound given in Theorem~\ref{thm:lb} because Theorem~\ref{thm:comp} holds only for the cases when eigenvalues of the kernel matrix follow a power law. In fact, an updated lower bound for kernel matrix with a skewed eigenvalue distribution is given in the following theorem.
\begin{theorem}\label{thm:lb1}
There exists a kernel matrix $K\in\mathbb R^{N\times N}$ with all its diagonal entries being $1$ and its eigenvalues following a $p$-power law such that for any sampling strategy that selects $m$ columns, the approximation error of the Nystr\"{o}m method is lower bounded by $\Omega(\frac{N}{m^p})$, i.e.,
\[
    \left\| K - K_b\Kh^{\dagger}K_b^{\top}\right\|_2 \geq \Omega\left(\frac{N}{m^p} \right),
\]
provided $N > 64[\ln 4]^2 m^2$.
\end{theorem}
We skip the proof of this theorem as it is almost identical to that of Theorem~\ref{thm:lb}. The gap between the upper bound and the lower bound given in Theorems~\ref{thm:comp} and \ref{thm:lb1} indicates that there is potentially a room for further improvement .

Next, we present several theorems and corollaries to pave the path for the proof of Theorem~\ref{thm:comp}. We borrow the following two theorems from the compressive sensing theory~\citep{Candes:2007:Sparsity} that are the key to our analysis.
\begin{theorem} (Theorem 1.2 from~\citep{Candes:2007:Sparsity})~\label{thm:9} Let $V$ be an $N\times N$ orthogonal matrix ($V^{\top}V = I$) with coherence $\mu$. Fix a subset $T$ of the signal domain. Choose a subset $S$ of the measurement domain of size $|S| = m$ uniformly at random. Suppose that the number of measurements $m$ obeys
$
m \geq |T| \mu^2 \max\left(C_a \ln |T|, C_b \ln(3/\delta)\right)
$
for some positive constants $C_a$ and $C_b$. Then
\[
\Pr\left(\left\|\frac{N}{m} V_{S, T}^{\top}V_{S, T} - I\right\|_2 \geq 1/2 \right) \leq \delta.
\]
\end{theorem}

\begin{theorem} (Lemma 3.3 from~\citep{Candes:2007:Sparsity})~\label{thm:9-1} Let $V$, $S$, and $T$ be the same as defined in Theorem~\ref{thm:9}. Let $\u_k^{\top}$ be the $k$-th row of $V^{\top}_{S, *}V_{S, T}$. Define $\sigma^2 = \mu^2 m \max\left(1, \mu |T|/\sqrt{m}\right)$. Fix $a > 0$ obeying $a \leq (m/\mu^2)^{1/4}$ if $\mu |T|/\sqrt{m} > 1$ and $a \leq (m/[\mu^2|T|])^{1/2}$ otherwise. Let $\z_k = (V^{\top}_{S, T} V_{S, T})^{-1} \u_k$. Then, we have
\begin{align*}
\Pr&\left(\sup\limits_{k \in T^c} \|\z_k\|_2 \geq 2\mu\sqrt{|T|/m} + 2a\sigma/m \right)\\
& \leq N\exp(-\gamma a^2) + \Pr\left(\|V_{S, T}^{\top} V_{S, T}\|_2 \leq \frac{m}{2N} \right)
\end{align*}
for some positive constant $\gamma$, where $T^c$ stands for the complementary set to $T$.
\end{theorem}

Combining the results from Theorem~\ref{thm:9} and Theorem \ref{thm:9-1}, we have the following high probability bound for $\sup_{k\in T^c}\|\z_k\|_2$.
\begin{corollary} \label{cor:3}
If $|T| \geq \max\left(C_{ab}\ln(3N^3), 4\frac{\ln N}{\gamma}\right)$, and $$ \mu^2\max\left(|T|C_{ab} \ln(3N^3), 16\left(\frac{\ln N}{\gamma}\right)^2\right)\leq m < \mu^2 |T|^2,$$
where $C_{ab}=\max(C_a, C_b)$, then with a probability $1 - 2N^{-3}$, we have
\[\displaystyle
    \sup\limits_{k \in T^c} \|\z_k\|_2 \leq 4 \mu\sqrt{\frac{ |T|}{m}}.
\]
\end{corollary}
%\begin{proof}
%We choose $a = 2\sqrt{\frac{\ln N}{\gamma}}$ in Theorem~\ref{thm:9-1}. Since $m\geq 16\mu^2\left(\frac{\ln N}{\gamma}\right)^2$, then we have $a\leq \left(\frac{m}{\mu^2}\right)^{1/4}$. And since $\mu|T|/\sqrt{m}>1$ holds,  therefor the conditions in Theorem~\ref{thm:9-1} hold, by settting $\delta=N^{-3}$ in Theorem~\ref{thm:9}, then the condition in Theorem~\ref{thm:9} hold, so we have
%\begin{eqnarray*}
%\Pr\left(\sup\limits_{k \in T^c} \|\z_k\|_2 \geq 2\mu\sqrt{|T|/m} + 2a\sigma/m \right) &\leq& N\exp(-\gamma a^2) + \Pr\left(\|V_{S, T}^{\top} V_{S, T}\|_2 \leq \frac{m}{2N} \right)\\
%&\leq & N^{-3} +  \Pr\left(\left\|\frac{N}{m}V_{S, T}^{\top} V_{S, T}-I\right\|_2 \geq \frac{1}{2} \right)\\
%&\leq& 2N^{-3}
%\end{eqnarray*}
%Then we have, with a probability $1 - 2N^{-3}$,
%\[
%\sup\limits_{k \in T^c}\|\z_k\|_2 \leq 2\mu\sqrt{\frac{|T|}{m}} + 2\left(\frac{m}{\mu^2}\right)^{1/4} \frac{\sqrt{\mu^{3}|T|m^{1/2}}}{m} = 4\mu\sqrt{\frac{|T|}{m}}
%\]
%\end{proof}
Using Corollary~\ref{cor:3}, we have the following bound for $\Er(\H_a, r)$.
\begin{theorem} \label{thm:10}
%Assume there exists a constant $\mu = O(1)$ such that $|V_{k,j}| \leq \mu/\sqrt{N}$ for any $1 \leq k, j \leq N$.
If $r > \max(C_{ab}\ln(3N^3), 4\ln N/\gamma)$ and
\[
    \mu^2\max\left(rC_{ab}\ln(3N^3), 16\left(\frac{\ln N}{\gamma}\right)^2 \right) \leq m <\mu^2 r^2,
\]
then, with a probability $1 - 2N^{-3}$, we have
\[
    \Er(\H_a, r) \leq \frac{16\mu^2 r}{m}\sum_{i=r+1}^N \lambda_i.
\]
%and hence
%\[\displaystyle
%    \left\| K - K_b\Kh^{-1}K_b \right\|_2 \leq \max\left(\Er(\H_a, r), \lambda_{r+1}\right) \leq \max\left(\frac{16\mu^2 r}{m}, 1\right)\sum_{i=r+1}^N \lambda_i.
%\]
\end{theorem}
\begin{proof}
For the sake of simplicity, we assume that the first $m$ examples are sampled, i.e., $\Dh = \{\x_1, \ldots, \x_m\}$.  For any $g \in \H_b^r$, we have
$
    g(\cdot) = \sum_{i=1}^r w_i\lambda_i^{1/2}\varphi_i(\cdot),
$
with $\sum_{i=1}^r w_i^2 \leq 1$. Below, we will make specific construction of $f$ based on $g$ that ensures a small approximation error. Let $f$ be
\begin{align*}
    f(\cdot) &= \sum_{j=1}^m a_j \kappa(\x_j, \cdot) = \sum_{i=1}^N \varphi_i(\cdot) \lambda_i^{1/2}\left(\sum_{j=1}^m a_j V_{j,i} \right) \\
    &= \sum_{i=1}^N b_i \lambda_i^{1/2}\varphi_i(\cdot),
\end{align*}
where
$
    b_i = \sum_{j=1}^m a_j V_{j,i}, i=1,\ldots, N
$, and the value of $\a=(a_1,\ldots, a_m)^{\top}$ will be given later.
Define $T = \{1, \ldots, r\}$ and $S = \{1, \ldots, m\}$. Under the condition that
\begin{align*}
    m &\geq r\mu^2\max\left(C_a, C_b)\right)\ln(3N^{3})\\
    &\geq r\mu^2\max\left(C_a\ln r, C_b\ln(3N^{3})\right),
\end{align*}
Theorem~\ref{thm:9} holds, and therefore  with a probability at least $1 - N^{-3}$,
\begin{eqnarray}
    \frac{m}{2N} \leq \lambda_{\min}\left(V_{S, T}^{\top}V_{S, T}\right) \leq \lambda_{\max}\left(V_{S, T}^{\top}V_{S, T}\right)  \leq \frac{3m}{2N}.
\end{eqnarray}
%In the analysis below, we assume that (\ref{temp:1}) holds.
We construct  $\a$ as
$
    \a = V_{S, T}\left[V_{S, T}^{\top}V_{S, T}\right]^{-1} \w
$,
where $\w = (w_1, \ldots, w_r)^{\top}$. Since
%\[
%V^{\top}_{S, T} \a = V^{\top}_{S, T}V_{S, T}\left[V_{S, T}^{\top}V_{S, T}\right]^{-1}\u = \u,
%\]
%we have $\lambda_i^{1/2}\gamma_i = u_i\lambda_i^{1/2} = w_i$ for $i=1, \ldots, r$. Also, note that
\[
    \mathbf b = V^{\top}_{S, *}\a=V^{\top}_{S,*}V_{S, T}\left(V_{S, T}^{\top}V_{S, T} \right)^{-1}\w,
\]
where $\mathbf b =(b_1,\cdots, b_N)^{\top}$, it is straightforward to see that  $b_j =  w_j$ for $j\in T$.
Using the result from Corollary~\ref{cor:3}, we have, with a probability at least $1 - 2N^{-3}$,
\[
    \max\limits_{j\in T^c} |b_j| \leq \max\limits_{j\in T^c} \|\z_j\|_2 \|\w\|_2 \leq 4\mu\sqrt{\frac{ r}{m}},
\]
where $\z_j^{\top}$ is the $j$-th row of matrix $V^{\top}_{S, *}V_{S, T}\left(V_{S, T}^{\top}V_{S, T} \right)^{-1}$.
We thus obtain
\[
    \|f - g\|_{\Hk}^2 = \left\|\sum_{i\in T^c} \lambda_i^{1/2}b_i \varphi_i(\cdot)\right\|_{\Hk}^2 \leq  \frac{16 \mu^2 r}{m} \sum_{i=r+1}^N \lambda_i.
\]
Hence,
\[
    \Er(\H_a, r) = \max\limits_{g \in \H_b^r} \min\limits_{f \in \H_a} \|f - g\|_{\Hk}^2 \leq \frac{16 \mu^2 r}{m} \sum_{i=r+1}^N \lambda_i.
\]
\end{proof}

\begin{remark} It is worthwhile to compare the result in Theorem~\ref{thm:10}, i.e., $\Er(\H_a, r) = O\left(\mu^2r\sum_{i=r+1}^N \lambda_i/m\right)$, to the relative error bound given in~\citep{journals/corr/abs-1110-5305}, i.e., $\Er(\H_a, r) \leq O\left(\lambda_{r+1} N/m\right)$. In the case when the eigenvalues decay fast (e.g., eigenvalues follow a power law), we have $\sum_{i=r+1}^N \lambda_i \ll N\lambda_{r+1}$, and therefore our bound is significantly better than the relative bound in~\citep{journals/corr/abs-1110-5305}. On the other hand, when eigenvalues follow a flat distribution (e.g., $\lambda_i \approx \lambda_{r+1}$ for all $i \in [r+2, N]$), we have $\sum_{i=r+1}^N \lambda_i \approx N\lambda_{r+1}$, and therefore our bound is worse than the relative bound in~\citep{journals/corr/abs-1110-5305} by a factor of $\mu^2 r$.
\end{remark}
Finally, we show the proof of Theorem~\ref{thm:comp} using Theorem~\ref{thm:10}.

%\begin{theorem}\label{cor:comp}
%Assume the eigenvalues follow a $p$-power law, i.e., there exists constants $0 < a \leq b$ and $p > 1$, such that $a N k^{-p} \leq \lambda_k \leq bN k^{-p}, k \in [N]$. We further assume there exists a constant $\mu = O(1)$ such that $V_{k,j}\leq \mu \sqrt{N}$ for $1 \leq k, j \leq N$. Given a sufficiently large number of samples (larger than the order of $O(\ln^2N)$), i.e.,
%\[
%    m > \mu^2\max\left(16\left(\frac{\ln N}{\gamma}\right)^2, 2C_{ab}\ln(3N^3), 4C^2_{ab}\ln^2 (3N^3)\right)
%\]
%we have, with a high probability $1 - 2N^{-3}$,
%\[
%    \left\|K - K_b\Kh^{-1}K_b^{\top} \right\|_2 \leq \widetilde{O}\left(\frac{N}{m^{p - 1}}\right).
%\]
%where $\widetilde{O}(\cdot)$ ignores the polynomial term of $\ln N$.
%\end{theorem}
\begin{proof}[Proof of Theorem~\ref{thm:comp}]
Let $\displaystyle r= \left\lfloor\frac{m}{\mu^2C_{ab}\ln (3N^3)}\right\rfloor$,
then
\[
\mu^2r C_{ab}\ln(3N^3) \leq m <\mu^2r^2,
\]
where the right inequality follows that $\displaystyle r\geq \frac{m}{2\mu^2C_{ab}\ln(3N^3)}$, and $m>4\mu^2C^2_{ab}\ln^2(3N^3)$. Then the conditions in Theorem~\ref{thm:10} hold and we have
\begin{align*}
\left\|K - K_b\Kh^{\dagger}K_b^{\top} \right\|_2 &\leq\max\left(\Er(\H_a, r), \lambda_{r+1}\right)  \\
&\leq \max\left(\frac{16\mu^2 r}{m}, 1\right)\sum_{i=r+1}^N \lambda_i.
\end{align*}
Since $\max(16\mu^2 r/m, 1)\leq O(1)$ due to the specific value we choose for $r$, and $\sum_{i=r+1}^N \lambda_i\leq O(N/r^{p-1})$ due to the power law distribution,
then
\[
\left\|K - K_b\Kh^{\dagger}K_b^{\top} \right\|_2 \leq O\left(\frac{N}{r^{p-1}}\right)\leq \widetilde{O}\left(\frac{N}{m^{p-1}}\right).
\]
\end{proof}
%The proof follows directly Theorem~\ref{thm:10} and the facts that $m = \tilde{O}(r)$ and $\sum_{i=r+1}^N \lambda_i = O(N/r^{p-1})$. %The following theorem shows that the optimal approximation error is $O(N/m^p)$, which to some degree justifies the near optimality of the bound $O(N/m^{p-1})$.

%\begin{theorem}
%Suppose the eigenvalues of $K$ follow a power law, i.e., there exists constants $0 < a < b$ and $p > 0$, such that
%\[
%    aNk^{-p} \leq \lambda_k \leq bNk^{-p}, k=1, \ldots, N
%\]
%Assume that $N \geq 2^{1/p}(m + 1)$. Then, for $\H_a$ constructed by any $m$ samples from $\D$, we have
%\[
%    \left| K - K_b \Kh^{-1} K_b^{\top}\right| \geq O\left(\frac{N}{m^p\ln N}\right)
%\]
%\end{theorem}
%\begin{proof}
%Recall the definition of $\Er(\H_a)$ is
%\[
%    \Er(\H_a) = \max\limits_{g \in \H_b} \Er(g, \H_a), \; \Er(g, \H_a) = \min\limits_{f \in \H_a} |f - g|_{\Hk}^2
%\]
%We construct $g(\cdot) = \sum_{k=1}^N w_k \varphi_k(\cdot)$ by setting $w_k = c\sqrt{N}k^{-(p+1)/2}, k \in [N]$. Since $\sum_{i=1}^N w_i^2/\lambda_i = 1$, it is easy to verify that $c = O(1/\sqrt{\ln N})$. Since the dimension of $\H_a$ is at most $m$, and therefore, for $g$ constructed as above, we have
%\[
%    \min\limits_{f \in \H_a} |f - g|_{\Hk}^2 \geq \sum_{k=m+1}^N w_k^2 = c^2 N\sum_{k=m+1}^N k^{-(p+1)} \geq \frac{c^2aN}{p}\left(\frac{1}{(m+1)^{p}} - \frac{1}{N^{p}} \right) \geq \frac{2c^2aN}{p(m+1)^{p}}
%\]
%\end{proof}

\section{Application of the Nystr\"{o}m Method to Kernel Classification}

Although the Nystr\"{o}m method was proposed in~\citep{Williams01usingthe} to speed up kernel machine, few studies examine the application of the Nystr\"{o}m method to kernel classification. In fact, to the best of our knowledge, ~\citep{Williams01usingthe} and~\citep{cortes-2010-nystrom} are the only two pieces of work that explicitly explore the Nystr\"{o}m method for kernel classification. The key idea of both works is to apply the Nystr\"{o}m method to approximate the kernel matrix with a low rank matrix in order to reduce the computational cost. More specifically, we consider the following optimization problem for kernel classification
\begin{eqnarray}
    \min\limits_{f \in \Hk} \L_N(f)=\frac{\lambda}{2}\|f\|_{\Hk}^2 + \frac{1}{N}\sum_{i=1}^N \ell(y_if(\x_i)) \label{eqn:svm},
\end{eqnarray}
where $y_i \in \{-1, +1\}$ is the class label assigned to instance $\x_i$, and $\ell(z)$ is a convex loss function. To facilitate our analysis, we assume (i) $\ell(z)$ is strongly convex with modulus $\sigma$, i.e. $|\ell''(z)| \geq \sigma$~\footnote{Loss functions such as square loss used for regression and logit function used for logistic regression are strongly convex}, and (ii) $\ell(z)$ is Lipschitz continuous, i.e. $|\ell'(z)| \leq C$ for any $z$ within the domain. Using the convex conjugate of the loss function $\ell(z)$, denoted by $\ell_*(\alpha), \alpha\in\Omega$, where $\Omega$ is the domain for dual variable $\alpha$, we can cast the problem in~(\ref{eqn:svm}) into the following optimization problem over $\alpha$

%\begin{eqnarray}
%    \min_{f\in\Hk}\max\limits_{\{\alpha_i \in \Omega\}_{i=1}^N} \frac{1}{N}\sum_{i=1}^N(\alpha_i y_if(\x_i) -  \ell_*(\alpha_i))   + \frac{\lambda}{2}\|f\|^2_{\Hk} \label{eqn:svm-primal-dual}
%\end{eqnarray}
\begin{eqnarray}
    \max\limits_{\{\alpha_i \in \Omega\}_{i=1}^N} - \frac{1}{N}\sum_{i=1}^N \ell_*(\alpha_i) - \frac{1}{2\lambda N^2} (\alpha \circ \y)^{\top} K (\alpha \circ \y) \label{eqn:svm-dual},
\end{eqnarray}
%Since the above problem is convex in $f$ and concave in $\alpha$, by switching min and max, and taking minimization of $f$ first, we have
with the solution $f$ given by
$
    f = -\frac{1}{N\lambda}\sum_{i=1}^N \alpha_iy_i \kappa(\x_i, \cdot)
$.
%and turn the variational optimization problem into the following maximization problem
%\begin{eqnarray}
%    \max\limits_{\{\alpha_i \in \Omega\}_{i=1}^N} - \frac{1}{N}\sum_{i=1}^N \ell_*(\alpha_i) - \frac{1}{2\lambda N^2} (\alpha \circ \y)^{\top} K (\alpha \circ \y) \label{eqn:svm-dual}
%\end{eqnarray}
By the Fenchel conjugate theory, we have $\max\limits_{\alpha \in \Omega} |\alpha|^2\leq C^2$.
because $|\ell'(z)|\leq C$.

To reduce the computational  cost, \cite{Williams01usingthe} and \cite{cortes-2010-nystrom}  suggest  to replace the kernel matrix $K$ with its low rank approximation $\Kt = K_b\Kh^{\dagger}K_b^{\top}$, leading to the following optimization problem for $\alpha$
\begin{eqnarray}
    \max\limits_{\{\alpha_i \in \Omega\}_{i=1}^N} - \frac{1}{N}\sum_{i=1}^N \ell_*(\alpha_i) - \frac{1}{2\lambda N^2} (\alpha \circ \y)^{\top} \Kt (\alpha \circ \y) \label{eqn:svm-1}.
\end{eqnarray}
One main problem with this approach is that although it simplifies the computation of kernel matrix, it does not simplify the classifier $f$,  because the number of support vectors, after the application of the Nystr\"{o}m method, is not guaranteed to be small~\citep{DBLP:conf/nips/DekelS06,Joachims:2009:SKS:1612900.1612908}, leading to a high computational cost in performing function evaluation.

We address this difficulty by presenting a new approach to explore the Nystr\"{o}m method for kernel classification. Similar to the previous analysis, we randomly select a subset of training examples, denoted by $\Dh = (\xh_1, \ldots, \xh_m)$, and restrict the solution of $f(\cdot)$ to the subspace $\H_a = \sp(\kappa(\xh_1, \cdot), \ldots, \kappa(\xh_m, \cdot))$, leading to the following optimization problem
\begin{eqnarray}
    \min\limits_{f \in \H_a}\L_N(f)= \frac{\lambda}{2}\|f\|_{\Hk}^2 + \frac{1}{N}\sum_{i=1}^N \ell(y_if(\x_i)) \label{eqn:svm-2}.
\end{eqnarray}
The following proposition shows that the optimal solution to~(\ref{eqn:svm-2}) is closely related to the optimal solution to (\ref{eqn:svm-1}).
\begin{proposition}\label{lem:equ} The solution $f$ to (\ref{eqn:svm-2}) is given by 
\[ \displaystyle f = -\frac{1}{N\lambda}\sum_{i=1}^m z_iy_i \kappa(\xh_i, \cdot),\] 
where $\z = \Kh^{\dagger}K_b^{\top}\alpha$ and $\alpha$ is the optimal solution to (\ref{eqn:svm-1}).
\end{proposition}
%\begin{proof}
%Since
%\[
%    \ell(y_if(\x_i)) = \max\limits_{\alpha_i \in \Omega} \alpha_iy_if(\x_i) - \ell_*(\alpha_i)
%\]
%we rewrite the optimization problem in (\ref{eqn:svm-2}) into a convex-concave optimization problem
%\[
%\min\limits_{f \in \H_a} \max\limits_{\{\alpha_i \in \Omega\}_{i=1}^m} \frac{\lambda}{2}\|f\|_{\Hk}^2 + \frac{1}{N}\sum_{i=1}^N \left(\alpha_i y_i f(\x_i) - \ell_*(\alpha_i)\right)
%\]
%Since $f \in \H_a$, we write $f = \sum_{i=1}^m z_i \kappa(\xh_i, \cdot)$, resulting in the following optimization problem
%\[
%\min\limits_{\z \in \R^m} \max\limits_{\{\alpha_i \in \Omega\}_{i=1}^m} \frac{\lambda}{2}\z^{\top}\Kh\z + \frac{1}{N}(\alpha \circ \y)^{\top}K_b \z  - \frac{1}{N}\sum_{i=1}^N \ell_*(\alpha_i)
%\]
%Since the above problem in linear (convex) in $\z$ and concave in $\alpha$, we can switch minimization with maximization. We complete the proof by taking the minimization over $\z$.
%\end{proof}

It is important to note that the classifier obtained from~(\ref{eqn:svm-2}) is only supported by the sampled training examples in $\Dh$, which significantly reduces the complexity of the kernel classifier compared to the approach suggested in~\citep{Williams01usingthe,cortes-2010-nystrom}. We also note that the proposed approach is equivalent to learning a linear classifier by representing each instance $\x$ with the vector
\[
    \phi(\x) = \widehat D^{-1/2}\widehat V^{\top}\left(\kappa(\xh_1, \x), \ldots, \kappa(\xh_m, \x) \right)^{\top},
\]
where $\widehat D$ is a diagonal matrix with non-zero eigenvalues of $\Kh$, and $\widehat V$
 is the corresponding eigenvector matrix. Although this idea has already been adopted by practitioners, we are unable to find any reference on its empirical study. The remaining of  this work is to show that this approach could have a good generalization performance provided that the eigenvalues of kernel matrix follow a skewed distribution. Below, we develop the generalization error bound for the classifier learned from (\ref{eqn:svm-2}).

Let $f_N$ and $f_N^a$ be the optimal solutions to (\ref{eqn:svm}) and (\ref{eqn:svm-2}),
respectively. Let $f^*$ be the optimal classifier that minimizes the expected loss function, i.e.,
\[
    f^* = \mathop{\arg\min}\limits_{f \in \Hk} P(\ell \circ f)\triangleq \E_{(\x, y)}\left[\ell(yf(\x)) \right].
\]
%We let $P_N(\ell \circ f)$ denote the empirical error
%\[
%P_N(\ell \circ f) = \frac{1}{N}\sum_{i=1}^N \ell(y_i f(\x_i)).
%\]
Let  $\|f\|_{L_2}^2 = \E_{\x}[|f(\x)|^2]$ denote the $\ell_2$ norm square of $f$.   In order to create a tight bound, we exploit the technique of local Rademacher complexity~\citep{Bartlett:2002:localrademacher, Koltchinskii:2011:oracle}. Define $\psi(\cdot)$ as
\begin{eqnarray*}
    \psi(\delta) = \left(\frac{2}{N}\sum_{i=1}^N \min(\delta^2, \lambda_i )\right)^{1/2}.
\end{eqnarray*}
Let $\et$ be the solution to $\et^2 = \psi(\et)$ where the existence and uniqueness of $\et$ is determined by the sub-root property of $\psi(\delta)$~\citep{Bartlett:2002:localrademacher}. Finally we define
\begin{eqnarray}
\epsilon = \max\left(\et, \sqrt{\frac{6\ln N}{N}}\right). \label{eqn:epsilon}
\end{eqnarray}

\begin{theorem}\label{thm:last}
Assume with a probability $1 - 2N^{-3}$, $\Er(\H_a) \leq \Gamma(N, m)$, where $\Gamma(N, m)$ is some function depending on $N$ and $m$. Assume that $N$ is sufficiently large such that
\begin{align*}
&\max\left(\|f_N^a\|_{\Hk},\|f^*\|_{\Hk}\right)\leq \frac{e^N N}{12\ln N},\\
&\max\left(\|f_N^a\|_{L_2}, \|f^*\|_{L_2}\right) \leq \frac{e^N}{2}\sqrt{\frac{N}{6\ln N}}.
\end{align*}
Then, with a probability at least $1 - 4N^{-3}$, we have
\begin{align*}
    P(\ell \circ f_N^a) &\leq P(\ell \circ f^*) + 2\lambda\|f^*\|_{\Hk}^2 + \frac{C^2\Gamma(N, m)}{\lambda N} \\
    &+ \frac{2C_1^2C^2\epsilon^4}{\lambda} + \frac{2C_1^2C^2\epsilon^2}{\sigma} + C_1 C e^{-N}
\end{align*}
where $\epsilon$ is given in (\ref{eqn:epsilon}) and $C_1$ is a constant independent from $m$ and $N$. By choosing $\lambda$ that minimizes the above bound, we have
\begin{align*}
    P(\ell \circ f_N^a) &\leq P(\ell \circ f^*) + 4\|f^*\|_{\Hk}\epsilon^2C\sqrt{C_1^2+ \frac{\Gamma(N, m)}{2N\epsilon^4}}\\
    &+ \frac{2C_1^2 C^2}{\sigma} \epsilon^2  + C_1 C e^{-N}.
\end{align*}

\end{theorem}

\begin{remark} In the case when the eigenvalues of the kernel matrix follow a $p$-power law with
$p>1$, we have $\epsilon^2 = O(N^{-p/(p+1)})$  according to~\citep{Koltchinskii:2010:mkl}, and
$\Gamma(N, m) = O(N/m^{p-1})$ according to Theorem~\ref{thm:comp}.  Applying these results to Theorem~\ref{thm:last},
the generalization performance of $f_N^a$ becomes
\begin{align}
    &P(\ell \circ f_N^a) \leq P(\ell \circ f^*) + 2\lambda\|f^*\|_{\Hk}^2 + \frac{C_2C^2}{\lambda m^{p-1}} + C_1 C e^{-N} \nonumber \\ &+ \frac{2C_3C^2N^{-2p/(p+1)}}{\lambda}+ \frac{2C_4C^2N^{-p/(p+1)}}{\sigma} \label{eqn:temp-bound-1}
\end{align}
where $C_2$, $C_3$, and $C_4$ are constants independent from $N$ and $m$.
By choosing $\lambda$ that minimizes the bound in (\ref{eqn:temp-bound-1}), we have
%it becomes
\begin{align*}
    P(\ell \circ f_N^a) & \leq  P(\ell \circ f^*)  + \frac{4\|f^*\|_{\Hk}}{N^{p/(p+1)}}C\sqrt{C_3 + C_2\frac{N^{2p/(p+1)}}{2m^{p-1}}}\\
    &\hspace*{0.2in}+ \frac{2C_4 C^2}{\sigma N^{p/(p+1)}} + C_1 C e^{-N}\\
    & =  P(\ell\circ f^*) + O\left(N^{-p/(p+1)} + m^{-(p - 1)/2} \right).
\end{align*}
\end{remark}
As indicated by above inequality, when the eigenvalues of the kernel matrix follow a $p$-power law, by setting $m = N^{2p/(p^2  - 1)}$, we are able to achieve similar performance as the full version of kernel classifier (i.e., $O(N^{-p/(p+1)})$). In other words, we can construct a kernel classifier without sacrificing its generalization performance with no more than $N^{2p/(p^2 - 1)}$ support vectors, which could be significantly smaller than $N$ when $p> (1+\sqrt{2} )$.  For the  example of kernel that generates Sobolev Spaces $W^{\alpha, 2}(\mathbb T^d)$ of smoothness $\alpha> d/2$, where $\mathbb T^d$ is $d$-dimensional torus,  its eigenvalues follow a $p$-power law with $p=2\alpha>d$, which is larger than $(1+\sqrt{2})$ when $d\geq 3$.

%We evaluate this  on cod\_RNA data set\footnote{\url{http://www.csie.ntu.edu.tw/~cjlin/libsvmtools/datasets/binary.html}}, which contains 59,535 instances for training and 271,617 instances for testing. We sample $m=0.05\%\sim 0.1\%$ examples for the proposed approach. The results are shown in Table~\ref{tab:1}, where we also report the results of linear SVM and RBF kernel SVM implemented by Liblinear and LibSVM, respectively.  The parameters  are always selected by five-fold cross-validation. The results verify that only a small number of samples (i.e. support vectors) is needed to achieve a descent performance.

%\begin{table}[t]
%\centering\caption{Performance on Cod\_RNA}\label{tab:1}\small{
%\begin{tabular}{lll|llllll}
%\toprule
%&Liblinear&Libsvm&\multicolumn{5}{c}{Proposed Nystr\"{o}m}\\
%sample size(\%)&100&100&0.05&0.06&0.07&0.08&0.09&0.1\\
%TrainTime(s)&3.1&80.2&3.48&3.83&4.35&4.72&5.13&5.73\\
%TestTime(s)&0.05&126.02&1.65&1.89&2.27&2.6&2.87&3.28\\
%Accuracy(\%)&70.71&96.67&82.51&83.9&84.32&84.68&86.43&87.23\\
% \bottomrule
%\end{tabular}}
%\end{table}
%
%t

\section{Conclusion}
We develop new methods for analyzing the approximation bound for the Nystr\"{o}m method. We show that the approximation error can be improved to $O(N/m^{1 - \rho})$ in the case when there is a large eigengap in the spectrum of  a kernel matrix, where $\rho \in (0, 1/2)$ is introduced to characterize the eigengap. When the eigenvalues of a kernel matrix follow a $p$-power law, the approximation error is further reduced to $O(N/m^{p-1})$ under an incoherence assumption. We develop a kernel classification approach based on the Nystr\"{o}m method and show that when the eigenvalues of a kernel matrix follow a $p$-power law ($p > 1$),  we can reduce the number of support vectors  to $N^{2p/(p^2 - 1)}$, which could be significantly less than $N$ if $p$ is large,  without seriously sacrificing its generalization performance.

\bibliography{nystrom}

\section*{Appendix}

\subsection*{Proof of Theorem~\ref{thm:lb}}
 We argue that there exists a kernel matrix $K$ such that (i) all its diagonal entries equal to $1$, and (ii) the first $m+1$ eigenvalues of $K$ are in the order of $\Omega(N/m)$. To see the existence of such a matrix, we sample $m+1$ vectors $\u_1, \cdots, \u_{m+1}$, where $\u_i \in \R^N$, from a Bernoulli distribution, with $\Pr(u_{i,j} = +1) = \Pr(u_{i,j} = -1) = 1/2$. We then construct $K$ as
\begin{align}\label{eqn:K}
K =\sum_{i=1}^{m+1}\u_i\u_i^{\top}\frac{1}{m+1}= \frac{1}{m+1} U U^{\top},
\end{align}
where $ U=(\u_1,\cdots, \u_{m+1})$.

First, since $u_{i,j} = \pm 1$, we have $diag(\u_i\u_i^{\top}) = \mathbf 1$, where $\mathbf 1$ is a vector of all ones, and therefore $K_{i,i} = 1$ for $i \in [N]$. Second, we show that with some probability $1 - \delta$, all non-zero eigenvalues of $\frac{1}{N}{ U}^{\top}  U$ are bounded between $1/2$ and $3/2$, i.e.,
\begin{eqnarray}
\frac{1}{2}\leq \lambda_{\min}\left(\frac{1}{N}{ U}^{\top} U\right) \leq \lambda_{\max}\left(\frac{1}{N}{ U}^{\top} U\right)\leq \frac{3}{2}. \label{eqn:bound}
\end{eqnarray}
To prove (\ref{eqn:bound}), we use the concentration inequality in Proposition~\ref{prop:conc}. We define $\xi_i = \z_i\z_i^{\top}, i=1,\ldots, N$, where $\z_i \in \R^m$ is the $i$th row of the matrix $ U$, and $\|\cdot\|$ in the above proposition as the spectral norm of a matrix. Since every element in $\z_i$ is sampled from a Bernoulli distribution with equal probabilities of being $\pm 1$, we have $\E[\z_i\z_i^{\top}] = \I_m$ and $\|\z_i\z_i^{\top}\| = m$. Thus, with a probability $1 - \delta$, we have
\[
    \left\|\frac{1}{N}{U}^{\top}{U} - \I \right\|=  \left\|\frac{1}{N}\sum_{i=1}^N\xi_i- \E[\xi] \right\|\leq \frac{4m\ln(2/\delta)}{\sqrt{N}}.
\]
When $N > 64 m^2[\ln 4]^2$, for any sampled $U$, with $50\%$ chance, we have
\[
  \left\|\frac{1}{N}{U}^{\top}{U} - \I \right\| \leq \frac{1}{2},
\]
which implies~(\ref{eqn:bound}).

With the bound in (\ref{eqn:bound}) and using the fact that the eigenvalues of $ U U^{\top}$ equal to the eigenvalues of $ U^{\top} U$, it is straightforward to see that the first $m+1$ eigenvalues of $K$ are in the order of $\Omega(N/m)$. Up to this point, we proved the existence of such a kernel matrix. Next, we prove the lower bound for the constructed kernel matrix.

Let $V_{1:(m+1)}=(\v_1,\cdots, \v_{m+1})$ the first $m+1$ eigenvectors of $K$.  We construct $\widehat g$ as follows: Let $\u= V_{1:(m+1)}\a$ be a vector in the subspace $\text{span}(\v_1,\cdots, \v_{m+1})$ that satisfies the condition $K_b^{\top}\u=0$. The existence of such a vector is guaranteed because $\text{rank}(K_b^{\top}V_{1:(m+1)})\leq m$.  We normalize $\a$ such that $\|\a\|_2=1$. Then we let $\widehat g= \sum_{i=1}^{N}u_i\kappa(\x_i,\cdot)=\sum_{i=1}^{m+1} w_i\sqrt{\lambda_{i}} \varphi_i(\cdot)$,  where $\w=V_{1:(m+1)}^{\top}\u$.
It is easy to verify that (i) $\widehat g\in \H_b$ since $\|\u\|_2= \|V_{1:(m+1)}\a\|_2=1$, and (ii) $\widehat g\perp \H_a$ since $\u^{\top}K_b=0$. Using $\gh$, we have
\begin{align*}
    \Er(\H_a) & = \max\limits_{g \in \H_b} \min\limits_{f \in \H_a} \|f - g\|_{\Hk}^2\geq  \|\widehat g\|_{\Hk}^2= \sum_{i=1}^{m+1}w_i^2\lambda_i\\
    &=\Omega\left(\frac{N}{m+1}\right)\|\w\|_2^2\geq \Omega\left(\frac{N}{m}\right),
\end{align*}
where we use $\|\w\|_2 = \|V_{1:(m+1)}^{\top}V_{1:(m+1)}\a\|_2=\|\a\|_2=1$. We complete the proof by using the fact $\Er(\H_a) = \left\| K - K_b\Kh^{\dagger}K_b^{\top}\right\|_2$.

%\section*{Appendix}
%\section{Proof of Theorem~\ref{thm:lb}}
%We construct the eigenvalues of $K$ as $\lambda_{i} = N/(m+1)$ for $i \leq m+1$ and zero otherwise.  Let $V_{1:(m+1)}=(\v_1,\cdots, \v_{m+1})$ the first $m+1$ othonormal eigenvectors.  We construct $\hat g$ as follows: Let $\u\in\text{span}(\v_1,\cdots, \v_{m+1})$, i.e. $\u= V_{1:(m+1)}\a$, satisfies $K_b^{\top}\u=0$,  which always exists since $\mathcal N(K_b^{\top}V_{1:(m+1)})\neq\emptyset$, due to $\text{rank}(K_b^{\top}V_{1:(m+1)})\leq m$.  We normalize $\a$ such that $\|\a\|_2=1$. Then we let $\hat g=\sum_{i=1}^{m+1} w_i\sqrt{\lambda_{i}} \varphi_i(\cdot)= \sum_{i=1}^{N}u_i\kappa(\x_i,\cdot)$,  where $\w=V_{1:(m+1)}^{\top}\u$.
%It is easy to verify that $\hat g\in \H_b$ since $\|\u\|_2= \|V_{1:(m+1)}\a\|_2=1$, and $\hat g\perp \H_a$ since $\u^{\top}K_b=0$. Hence
%\[
%    \Er(\H_a) \geq \max\limits_{g \in \H_b} \min\limits_{f \in \H_a} \|f - g\|_{\Hk}^2\geq  \|\hat g\|_{\Hk}^2= \sum_{i=1}^{m+1}w_i^2\lambda_i=\frac{N}{m+1}\|\w\|_2^2=\frac{N}{m+1}\geq \Omega\left(\frac{N}{m}\right),
%\]
%where $\|\w\|_2 = \|V_{1:(m+1)}^{\top}V_{1:(m+1)}\a\|_2=\|\a\|_2=1$. We complete the proof by using the fact $\Er(\H_a) = \left\| K - K_b\Kh^{-1}K_b^{\top}\right\|_2$.

\subsection*{Proof of Corollary~\ref{cor:1}}
Define $\xi(\xh_i)$ to be a rank one linear operator, i.e.,
\[
    \xi(\xh_i)[f](\cdot) = \kappa(\xh_i, \cdot) f(\xh_i).
\]
Apparently, $L_m = \frac{1}{m}\sum_{i=1}^m \xi(\xh_i)$ and $\mathrm{E}[\xi(\xh_i)] = L_N$. We complete the proof by using the result from Proposition~\ref{prop:conc} and the fact
\begin{align*}
\|\xi(\xh_k)\|_{HS}&= \sqrt{\sum_{i,j=1}^N\langle\varphi_i, \kappa(\xh_k,\cdot)\varphi_j(\xh_k) \rangle^2}\\
&=\sqrt{\sum_{i,j=1}^N\varphi_i(\xh_k)^2\varphi_j(\xh_k)^2}= \kappa(\xh_k, \xh_k)\leq 1,
\end{align*}
where the last equality follows equation~(\ref{eqn:keig}).

%\subsection*{Proof of Corollary~\ref{cor:diff}}
%\begin{proof}
%Let $(\mu_i, \omega_i), i \in [N]$ be the eigenvalues and eigenfunctions of $L_N - L_m$. According to Corollary~\ref{cor:1}, we have, with a probability at least $1 - \delta$,
%\[
%    \sqrt{\sum_{i=1}^N \mu_i^2} \leq \frac{4\ln(2/\delta)}{\sqrt{m}}
%\]
%Since
%\[
%\max\limits_{\|h\|_{\Hk} \leq 1} \|(L_N - L_m)h\|^2_{\Hk} = \max\limits_{\u \in \R^N, \|\u\|_2 \leq 1} \sum_{i} \mu_i^2 u_i^2 \leq \sum_{i} \mu_i^2,
%\]
%we have the result.
%\end{proof}

%\subsection*{Proof of Corollary~\ref{cor:2}}
%Using Theorem~\ref{thm:6}, we have, with a probability at least $1 - \delta$,
%\begin{eqnarray}
%\left\| K - K_b\Kh^{-1} K_b^{\top}\right\|_2 \leq \frac{16[\ln(2/\delta)]^2 N^2}{m(\lambda_r - \lambda_N)} + \lambda_{r+1} - \lambda_N + \lambda_N, r = 1, \ldots, N-1 \label{temp:1}
%\end{eqnarray}
%Let $\lambda$ be any number between $\lambda_{r+1} - \lambda_N$ and $\lambda_r - \lambda_N$. Evidently, we have,
%\begin{eqnarray}
%\left\| K - K_b\Kh^{-1} K_b^{\top}\right\|_2 \leq \frac{16[\ln(2/\delta)]^2 N^2}{m\lambda} + \lambda + \lambda_N, \lambda \in [\lambda_{r+1} - \lambda_N, \lambda_r - \lambda_N] \label{temp:2}
%\end{eqnarray}
%Combining (\ref{temp:1}) with (\ref{temp:2}), we have, with a probability at least $1 - \delta$,
%\begin{eqnarray*}
%\left\| K - K_b\Kh^{-1} K_b^{\top}\right\|_2 \leq \frac{16[\ln(2/\delta)]^2 N^2}{m\lambda} + \lambda + \lambda_N, \lambda \in [0, \lambda_1 - \lambda_N]
%\end{eqnarray*}
%We complete the proof by minimizing the R.H.S over $\lambda$.

\subsection*{Proof of Corollary~\ref{cor:3}}
We choose $a = 2\sqrt{{\ln N}/{\gamma}}$ in Theorem~\ref{thm:9-1}. Since $m\geq 16\mu^2\left(\frac{\ln N}{\gamma}\right)^2$, then we have $a\leq \left(\frac{m}{\mu^2}\right)^{1/4}$. Additionally,  by having $\mu|T|/\sqrt{m}>1$,  the conditions in Theorem~\ref{thm:9-1} hold, and by setting $\delta=N^{-3}$ in Theorem~\ref{thm:9}, the condition in Theorem~\ref{thm:9} holds, which together implies 
\begin{align*}
&\Pr\left(\sup\limits_{k \in T^c} \|\z_k\|_2 \geq 2\mu\sqrt{|T|/m} + 2a\sigma/m \right)\\
 &\leq N\exp(-\gamma a^2) + \Pr\left(\|V_{S, T}^{\top} V_{S, T}\|_2 \leq \frac{m}{2N} \right)\\
&\leq  N^{-3} +  \Pr\left(\left\|\frac{N}{m}V_{S, T}^{\top} V_{S, T}-I\right\|_2 \geq \frac{1}{2} \right)\\
&\leq 2N^{-3}.
\end{align*}
From this we have, with a probability $1 - 2N^{-3}$,
\begin{align*}
\sup\limits_{k \in T^c}\|\z_k\|_2 &\leq 2\mu\sqrt{\frac{|T|}{m}} + 2\left(\frac{m}{\mu^2}\right)^{1/4} \frac{\sqrt{\mu^{3}|T|m^{1/2}}}{m}\\
& = 4\mu\sqrt{\frac{|T|}{m}}.
\end{align*}

\subsection*{Proof of Proposition~\ref{lem:equ}}
Since
\[
    \ell(y_if(\x_i)) = \max\limits_{\alpha_i \in \Omega} \alpha_iy_if(\x_i) - \ell_*(\alpha_i),
\]
we rewrite the optimization problem in (\ref{eqn:svm-2}) into a convex-concave optimization problem
\[
\min\limits_{f \in \H_a} \max\limits_{\{\alpha_i \in \Omega\}_{i=1}^m} \frac{\lambda}{2}\|f\|_{\Hk}^2 + \frac{1}{N}\sum_{i=1}^N \left(\alpha_i y_i f(\x_i) - \ell_*(\alpha_i)\right).
\]
Since $f \in \H_a$, we write $f = \sum_{i=1}^m z_i \kappa(\xh_i, \cdot)$, resulting in the following optimization problem
\[
\min\limits_{\z \in \R^m} \max\limits_{\{\alpha_i \in \Omega\}_{i=1}^m} \frac{\lambda}{2}\z^{\top}\Kh\z + \frac{1}{N}(\alpha \circ \y)^{\top}K_b \z  - \frac{1}{N}\sum_{i=1}^N \ell_*(\alpha_i).
\]
Since the above problem in linear (convex) in $\z$ and concave in $\alpha$, we can switch minimization with maximization. We complete the proof by taking the minimization over $\z$.

\subsection*{Proof of Theorem~\ref{thm:last}}
To simply our presentation, we introduce notations
\begin{align*}
&P_N(\ell\circ f)=\frac{1}{N}\sum_{i=1}^N\ell(y_if(\x_i)),\\
&\Lambda(f) = P(\ell\circ f) - P(\ell\circ f^*).
\end{align*}
Using $P_N(\ell\circ f)$, we can write $\L_N(f) =P_N(\ell\circ f) + \frac{\lambda}{2}\|f\|^2_{\Hk}$.
We first  prove that
\begin{equation*}
\L_N(f_N)\leq \L_N(f_N^a) +  \frac{C^2}{2\lambda N}\Er(\H_a),
\end{equation*}
where $\max_{z\in\Omega}|z|^2\leq C^2$. Note that
\begin{align*}
  &\L_N(f_N) \\
  &= \max\limits_{\{\alpha_i \in \Omega\}_{i=1}^N} - \frac{1}{N}\sum_{i=1}^N \ell_*(\alpha_i) - \frac{1}{2\lambda N^2} (\alpha \circ \y)^{\top} K (\alpha \circ \y)\\
    &\L_N(f^a_N)\\
    &= \max\limits_{\{\alpha_i \in \Omega\}_{i=1}^N} - \frac{1}{N}\sum_{i=1}^N \ell_*(\alpha_i) - \frac{1}{2\lambda N^2} (\alpha \circ \y)^{\top} \widetilde K (\alpha \circ \y).
\end{align*}
Then
\begin{align*}
 & \L_N(f_N)\\
 & = \max\limits_{\{\alpha_i \in \Omega\}_{i=1}^N} - \frac{1}{N}\sum_{i=1}^N \ell_*(\alpha_i) - \frac{1}{2\lambda N^2} (\alpha \circ \y)^{\top} \widetilde K (\alpha \circ \y) \\
 &\hspace*{0.2in}+ \frac{1}{2\lambda N^2} (\alpha \circ \y)^{\top} (\widetilde K - K)(\alpha \circ \y)\\
&\leq \max\limits_{\{\alpha_i \in \Omega\}_{i=1}^N} - \frac{1}{N}\sum_{i=1}^N \ell_*(\alpha_i) - \frac{1}{2\lambda N^2} (\alpha \circ \y)^{\top} \widetilde K (\alpha \circ \y)\\
&\hspace*{0.2in} +\max\limits_{\{\alpha_i \in \Omega\}_{i=1}^N}\frac{1}{2\lambda N^2} (\alpha \circ \y)^{\top} (\widetilde K - K)(\alpha \circ \y)\\
&\leq \L_N(f^a_N) + \frac{1}{2\lambda N^2}\|\alpha\|_2^2 \|K-\widetilde K\|_2\\
&\leq \L_N(f^a_N) + \frac{C^2}{2\lambda N} \Er(\H_a).
  \end{align*}
Then we proceed the proof as follows
\begin{align*}
    &\frac{\lambda}{2} \|f_N^a\|_{\Hk}^2 + P(\ell \circ f^a_N)\\
     & \leq  P_N(\ell \circ f^a_N) + \frac{\lambda}{2} \|f_N^a\|_{\Hk}^2 + (P - P_N)(\ell \circ f^a_N) \\
    & \leq  P_N(\ell \circ f_N) + \frac{\lambda}{2} \|f_N\|_{\Hk}^2 + \frac{C^2}{2\lambda N}\Er(\H_a) \\
    &\hspace*{0.1in}+ (P - P_N)(\ell \circ f^a_N) \\
    & \leq  P_N(\ell \circ f^*) + \frac{\lambda}{2} \|f^*\|_{\Hk}^2 + \frac{C^2}{2\lambda N}\Er(\H_a)\\
    &\hspace*{0.1in} + (P - P_N)(\ell \circ f^a_N),
\end{align*}
%where the second inequality follows from Proposition~\ref{prop:3}, and
where the third inequality follows from the fact that $f_N$ is the minimizer of $P_N(\ell\circ f) + \frac{\lambda}{2}\|f\|_{\Hk}^2$. Hence,
\begin{align*}
    \Lambda(f_N^a) & \leq  \frac{\lambda}{2}\|f^*\|_{\Hk}^2 - \frac{\lambda}{2} \|f_N^a\|_{\Hk}^2 + \frac{C^2}{2\lambda N}\Er(\H_a) \\
    &+ (P - P_N)(\ell \circ f^a_N - \ell \circ f^*).
\end{align*}
Let $r = \|f^* - f^a_N\|_{L_2}$ and $R = \|f^* - f^a_N\|_{\Hk}$. Define
\[
    \G(r, R) = \left\{f \in \Hk: \|f - f^*\|_{L_2} \leq r, \|f^* - f\|_{\Hk} \leq R \right\}.
\]
Using the domain $\G$, we rewrite the bound for $\Lambda(f^a_N)$ by
\begin{align*}
    \Lambda(f_N^a) &\leq \frac{\lambda}{2}\|f^*\|_{\Hk}^2 - \frac{\lambda}{2} \|f_N^a\|^2_{\Hk} + \frac{C^2}{2\lambda N}\Er(\H_a) \\
    &+ \sup\limits_{f \in \G(r, R)} (P - P_N)(\ell \circ f - \ell \circ f^*).
\end{align*}
Since $\epsilon r \leq e^N$ and $\epsilon^2R \leq e^N$
, using Lemma 9 from~\citep{Koltchinskii:2010:mkl}, we have, with a probability $1 - 2N^{-3}$, for any \begin{eqnarray*}
\sup_{f \in \G(r, R)} (P - P_N)(\ell \circ f - \ell \circ f^*)) \leq C_1C(r\epsilon + R\epsilon^2 + e^{-N}),
\end{eqnarray*}
where $C_1$ is a constant independent from $N$. Thus, with a probability at least $1 - 4N^{-3}$, we have
\begin{align*}
& \Lambda(f_N^a) - C_1 C e^{-N}  \\
& \leq  \frac{\lambda}{2}\|f^*\|_{\Hk}^2 - \frac{\lambda}{2}\|f_N^a\|_{\Hk}^2 + \frac{C^2\Gamma(N, m)}{2\lambda N}\\
& \hspace*{0.2in}+ C_1C\epsilon \|f_N^a - f^*\|_{L_2} + C_1C\epsilon^2\|f^* - f^a_N\|_{\Hk} \\
& \leq  \frac{\lambda}{2}\|f^*\|_{\Hk}^2 - \frac{\lambda}{2}\|f_N^a\|_{\Hk}^2 + \frac{C^2\Gamma(N, m)}{2\lambda N}\\
&\hspace*{0.1in} + \frac{C_1^2C^2\epsilon^2}{\sigma} + \frac{\sigma}{4} \|f_N^a - f^*\|_{L_2}^2 + \frac{C_1^2C^2\epsilon^4}{\lambda} + \frac{\lambda}{4}\|f^* - f_N^a\|_{\Hk}^2 \\
& \leq  \frac{\lambda}{2}\|f^*\|_{\Hk}^2 - \frac{\lambda}{2}\|f_N^a\|_{\Hk}^2 + \frac{C^2\Gamma(N, m)}{2\lambda N}+ \frac{\lambda}{2}\|f^*\|_{\Hk}^2  \\
&\hspace*{0.1in}+ \frac{C_1^2C^2\epsilon^2}{\sigma} + \frac{\sigma}{4} \|f_N^a - f^*\|_{L_2}^2 + \frac{C_1^2L^2\epsilon^4}{\lambda} +\frac{\lambda}{2} \|f_N^a\|_{\Hk}^2 \\
& \leq  \lambda\|f^*\|_{\Hk}^2 + \frac{C^2\Gamma(N, m)}{2\lambda N} + \frac{C_1^2C^2\epsilon^2}{\sigma}  + \frac{C_1^2C^2\epsilon^4}{\lambda} \\
&\hspace*{0.1in}+ \frac{\sigma}{4} \|f_N^a - f^*\|_{L_2}^2\\
& \leq  \lambda\|f^*\|_{\Hk}^2 + \frac{C^2\Gamma(N, m)}{2\lambda N} + \frac{C_1^2C^2\epsilon^2}{\sigma} + \frac{C_1^2C^2\epsilon^4}{\lambda} + \frac{1}{2}\Lambda(f^a_N),
\end{align*}
where in the second inequality we apply Young's inequality  $ab\leq \frac{a^2}{2\epsilon} + \frac{\epsilon b^2}{2}$ twice, the last inequality follows from the strong convexity of $\ell(\z)$ and $f^*$ is the minimizer of $P(\ell\circ f)=\E_{(\x, y)}[\ell(yf(\x))]$. Thus, with a probability at least $1 - 4N^{-3}$, we have
\begin{align*}
    P(\ell \circ f_N^a) \leq& P(\ell \circ f^*) + 2\lambda\|f^*\|_{\Hk}^2 + \frac{C^2\Gamma(N, m)}{\lambda N} \\
    &+ \frac{2C_1^2C^2\epsilon^2}{\sigma} + \frac{2C_1^2C^2\epsilon^4}{\lambda}+ C_1 C e^{-N}.
\end{align*}
We complete the proof by minimizing over $\lambda$ in the R.H.S. of the above inequality.

\end{document}